\documentclass{article}

\newif\ifarXiv
\arXivtrue

\PassOptionsToPackage{sort, numbers, compress} {natbib}

\usepackage{subcaption}

\usepackage{iclr2027_conference}
\usepackage{times}

\ifarXiv
    \iclrfinalcopy
\fi

\usepackage[utf8]{inputenc} 
\usepackage[T1]{fontenc}    
\usepackage{hyperref}       
\usepackage{url}            
\usepackage{booktabs}       
\usepackage{amsfonts}       
\usepackage{framed}
\usepackage{nicefrac}       
\usepackage{microtype}      
\usepackage{xcolor}         
\usepackage{babel}
\usepackage{paralist}
\usepackage{enumitem}
\usepackage{siunitx}
\usepackage{tikz}
\usepackage{mathtools}
\usepackage{wrapfig}
\usepackage{silence}

\usepackage{amsmath}
\usepackage{amsthm}
\usepackage{thmtools}
\usepackage{mathrsfs}
\usepackage[dvipsnames]{xcolor}

\definecolor{OIOrange}    {RGB}{230, 159,   0}
\definecolor{OISkyBlue}   {RGB}{ 86, 180, 233}
\definecolor{OIGreen}     {RGB}{  0, 158, 115}
\definecolor{OIYellow}    {RGB}{240, 228,  66}
\definecolor{OIBlue}      {RGB}{  0, 114, 178}
\definecolor{OIVermillion}{RGB}{213,  94,   0}
\definecolor{OIRose}      {RGB}{204, 121, 167}

\hypersetup{
    colorlinks = true,
    urlcolor   = OIBlue,
    citecolor  = OIBlue,
    linkcolor  = OIBlue,
}
 
\newtheorem{theorem}{Theorem}

\newtheorem{definition}{Definition}
\newtheorem{assumption}{Assumption}

\newcommand{\R}{\mathbb{R}}
\newcommand{\Sd}{S^{d-1}}
\newcommand{\ECT}{\mathrm{ECT}}
\newcommand{\vertK}{\mathrm{vert}}
\newcommand{\coface}{\mathrm{coface}}
\newcommand{\1}{\mathbf{1}}
\newcommand{\vx}{\mathbf{x}}

\newcommand{\inlinetriangle}[2][0]{%
  \tikz[baseline=#2, rotate=#1]{
    \coordinate (A) at (0,0);
    \coordinate (B) at (0.4em,1.6ex);
    \coordinate (C) at (0.8em,0);
    \draw (A) -- (B) -- (C) -- cycle;
    \filldraw (A) circle (0.8pt)
              (B) circle (0.8pt)
              (C) circle (0.8pt);
  }%
}

\title{Geometry-Aware Simplicial Message Passing}

\author{%
  Elena Xinyi Wang\\
  AIDOS Lab, University of Fribourg\\
  Fribourg, Switzerland\\
  \url{xinyi.wang@unifr.ch} \\
  \And
  Bastian Rieck\\
  AIDOS Lab, University of Fribourg\\
  Fribourg, Switzerland\\
  \url{bastian.grossenbacher@unifr.ch} \\
}

\begin{document}

\maketitle

\ifarXiv
    \lhead{Preprint. Under review.}
\fi

\begin{abstract}
The Weisfeiler–Lehman (WL) test and its simplicial extension (SWL) characterize the combinatorial expressivity of message passing networks, but they are blind to geometry, i.e., meshes with \emph{identical} connectivity but \emph{different} embeddings are indistinguishable. 
We introduce the \textsc{Geometric Simplicial Weisfeiler–Lehman}~(GSWL) test, which incorporates vertex coordinates into color refinement for geometric simplicial complexes.
In addition, we show that
\begin{inparaenum}[(i)]
    \item the expressivity of geometry-aware simplicial message passing schemes is bounded above by GSWL, and
    \item that there exist parameters such that the discriminating power of GSWL is matched by these schemes on any fixed finite family of geometric simplicial complexes.
\end{inparaenum}
Combined with the Euler Characteristic Transform~(ECT), a complete invariant for geometric simplicial complexes, this yields a geometric expressivity characterization together with an approximation framework. 
Experiments on synthetic and mesh datasets illustrate our theory, showing a clear hierarchy from combinatorial to geometry-aware models.
\end{abstract}
 
\section{Introduction}
 
The expressive power of graph neural networks (GNNs) is typically understood through the lens of the Weisfeiler--Lehman~(WL) test for graph isomorphism~\citep{xu2019powerful,morris2019weisfeiler, Morris23a}. 
Message passing GNNs with injective aggregation match WL in their ability to distinguish non-isomorphic graphs, and this characterization has guided architectural design across the field.
Bodnar et al.~\citep{bodnar2021weisfeiler} extended this framework to simplicial complexes, introducing the Simplicial WL~(SWL) test and Message Passing Simplicial Networks~(MPSNs). 
With a clique-complex lifting, SWL is \emph{strictly more powerful} than \mbox{$1$-WL} and \emph{no less powerful} than \mbox{$3$-WL}. 
This provides a combinatorial expressivity theory for higher-order message passing. However, SWL and MPSNs operate on the abstract simplicial complex alone and are thus entirely blind to the geometric realization: If two meshes share the same vertex-edge-triangle incidence structure but carry different vertex coordinates, SWL produces \emph{identical} colorings for both.
Our work directly addresses this gap.
 
A parallel line of work on geometric GNNs~\citep{joshi2023expressive,han2022geometrically} studies expressivity under Euclidean group actions.
Closest to our setting, $E(n)$-equivariant simplicial networks~\citep{eijkelboom2023en} are equivariant under Euclidean motions; with an invariant readout, inputs related by such motions are intentionally identified.
We instead target the finer, coordinate-preserving simplicial isomorphism in a fixed frame. 
To characterize the \emph{geometric expressivity} of simplicial message passing, we therefore need an invariant that captures embedding information and is ideally complete, so that matching it loses no information about the geometric support.
The Euler Characteristic Transform~(ECT) is a natural fit~\citep{munch2025ect, turner2014persistent, curry2022many,ghrist2018persistent, Rieck25a}.
It maps an embedded complex to a function on directions and thresholds by tracking the Euler characteristic of sublevel sets, and is a \emph{complete invariant} for the geometric support of embedded simplicial complexes. 
Thus, the ECT has been widely used as a topological descriptor for shape analysis in, among others, plant biology, histopathology, and graph learning~\citep{ECT2021barley,tang2022ectTumor, roell2024differentiable}.

\paragraph{Our contributions.}
We use the ECT to characterize the expressivity of simplicial message passing for \emph{embedded} simplicial complexes. In our setting, \emph{equivalence} means equality of the embedded simplicial complex up to relabeling. 
Our main contributions are:
\begin{enumerate}[left = 0pt, itemsep=0.2em, topsep=0.2em]
\item \textbf{Geometric Simplicial Weisfeiler--Lehman (GSWL):} A geometric version of SWL that is geometry-aware and refines color through boundary and coboundary adjacencies~(Section~\ref{sec:gswl}).

\item \textbf{Message-passing characterization:} We show that geometry-aware simplicial message passing is bounded above by GSWL and can match GSWL on any fixed finite family~(Theorems~\ref{thm:upper}--\ref{thm:lower}). Moreover, combining the lower bound with coordinate recovery results in exact computation of sampled ECT values on finite families~(Theorem~\ref{thm:ect}).
Separately, coordinate recovery gives an error bound on the approximation of the full ECT on bounded classes~(Corollary~\ref{cor:stability}) via a prior stability result~\citep{george2025stability}.

\item \textbf{Experiments:} We illustrate the predicted combinatorial-to-geometric hierarchy on synthetic meshes, manifold triangulations~\citep{mantra2025}, and registered human body meshes~\citep{bogo2014faust}, with a full baseline hierarchy including GCN, GIN, DeepSets, and combinatorial SMP (Section~\ref{sec:experiments}).
\end{enumerate}

Since the ECT is a complete invariant of the geometric support, this establishes that geometry-aware simplicial message passing can realize sampled values of a complete geometric invariant on finite families and, by stability, approximate the full invariant on bounded classes.
GSWL itself is complete for the finer coordinate-preserving isomorphism.
In this sense, GSWL plays the role for embedded simplicial complexes that WL plays for graphs, with the ECT supplying its geometric content.
 
\paragraph{Related work.}
The WL test~\citep{weisfeiler1968reduction} and its connection to GNN expressivity~\citep{xu2019powerful,morris2019weisfeiler} is a foundational paradigm of graph learning research~\citep{Morris23a}.
Higher-order $k$-WL variants~\citep{maron2019provably,morris2020weisfeiler} trade locality for power. 
\citet{bodnar2021weisfeiler} introduced SWL and MPSNs for simplicial complexes, proving $\text{SWL} > \text{WL}$ and $\text{SWL} \geq \text{3-WL}$ in terms of their expressivity. 
Our GSWL reuses their adjacency structure but with coordinate-derived initial features, which is a strict refinement.
The Euler Characteristic Transform~(ECT) was introduced by \citet{turner2014persistent} and shown to be injective on a broad class of shapes. 
\citet{curry2022many} and \citet{ghrist2018persistent} proved injectivity for finitely many directions.\footnote{
    Despite their different publication dates, both of these publications appeared in parallel as preprints.
}
\citet{crawford2020predicting} and \citet{roell2024differentiable} developed differentiable ECT layers for machine learning, but without the expressivity-theoretic connection to message passing that we establish here.
Explicit connections to message passing with local variants of the ECT were sketched by \citet{vonrohrscheidt2025diss} and pursued further by \citet{amboage2026leap}.
Adopting a complementary perspective, geometric GNNs~\citep{joshi2023expressive,han2022geometrically,villar2021scalars,brandstetter2022geometric} characterize expressivity under Euclidean symmetries for point clouds and graphs. 
Both perspectives are now commonly subsumed in the field of \emph{topological deep learning}~\citep{papamarkou2024position}, which provides a general framework, among other things, message passing on cell and simplicial complexes, but typically without expressivity characterizations tied to geometric invariants.
Notably, \citet{akbari2026logical} characterize TNN expressivity on abstract combinatorial complexes, whereas our setting is geometric and thus complementary.

\begin{framed}
Our work differs from existing work in two ways: We consider \emph{simplicial complexes} (not just graphs) and our expressivity target is \emph{embedding equivalence}~(not equivalence up to group action), i.e.,  in our setting, \inlinetriangle[0]{-0.5ex} and
\inlinetriangle[180]{-2.0ex} are non-equivalent simplicial complexes.
\end{framed}

\section{Background}\label{sec:background}

We review the WL test for graphs, its simplicial extension SWL, and the Euler Characteristic Transform. 
These are the three components our construction builds on.
 
\subsection{The Weisfeiler--Lehman~(WL) test and GNNs}
 
The $1$-WL test~\citep{weisfeiler1968reduction} iteratively refines vertex colors by hashing each vertex's color together with the multiset of its neighbors' colors, i.e.,
\begin{equation}
c_v^{(t+1)} = \mathrm{HASH}\!\left(c_v^{(t)},\; \{\!\{c_u^{(t)} : u \in \mathcal{N}(v)\}\!\}\right).
\end{equation}
\citet{xu2019powerful} and \citet{morris2019weisfeiler} proved that message-passing GNNs with injective aggregation are exactly as powerful as $1$-WL. 
This characterization has been central to understanding and improving GNN architectures~\citep{Morris23a}.
 
\subsection{Simplicial complexes and the simplicial Weisfeiler--Lehman~(SWL) test}
 
A finite abstract simplicial complex $K$ on a vertex set $V$ is a collection of nonempty subsets of $V$ that is closed under taking nonempty subsets~\citep{hatcher}. 
A $k$-simplex $\sigma = \{v_0, \ldots, v_k\}$ has dimension $k$. 
We write $\partial \sigma$ for the set of codimension-1 faces (boundary) and $\coface(\sigma)$ for the set of simplices of dimension $\dim(\sigma)+1$ containing $\sigma$ (coboundary).
\citet{bodnar2021weisfeiler} defined the Simplicial WL (SWL) test by initializing all simplices of the same dimension with the same color and refining via boundary, coboundary, lower-adjacent, and upper-adjacent neighborhoods.
They show that boundary and upper adjacencies suffice to reach the converged SWL partition.  
Here, we work with boundary and coboundary adjacencies,\footnote{
We recall that two simplices are boundary-adjacent when one is a face of the other, and coboundary-adjacent when one is a coface of the other. Thus, $\sigma$ aggregates from $\partial\sigma$~(one dimension below) and from $\coface(\sigma)$~(one dimension above).
} 
which yield the \emph{same} stable partition as full SWL when started from the same initialization, although the finite-depth partitions can differ.
Specifically, our coordinate-recovery argument~(Lemma~\ref{lem:coord-recovery}) tracks how vertex information propagates up the Hasse diagram one dimension at a time, and this propagation rate underlies the depth bound $L\ge\dim(K)$ in Theorem~\ref{thm:ect}.

A finite geometric simplicial complex in $\R^d$ can be represented by a pair $(K, \vx)$, where $K$ is a finite abstract simplicial complex and $\vx\colon \vertK(K) \to \R^d$ extends affinely on each simplex to an injective map $|K|\to\mathbb{R}^d$.
We also refer to $(K,\vx)$ as an embedded simplicial complex.
Its geometric support is
$|K,x|
  :=\bigcup_{\sigma\in K}
    \operatorname{conv}\{x_v:v\in\sigma\}$.
    
For a simplex $\sigma = \{v_0, \ldots, v_k\}$, we define coordinate-derived features:
\begin{equation}
\begin{split}
\text{Vertices:} &\quad \vx_v \in \R^d, \\
\text{Edges } \{u,v\}: &\quad \tfrac{1}{2}(\vx_u + \vx_v),\; \|\vx_u - \vx_v\|, \\
\text{Triangles } \{u,v,w\}: &\quad \tfrac{1}{3}(\vx_u + \vx_v + \vx_w),\; \mathrm{area}(\{u,v,w\}).
\end{split}
\end{equation}
Two embedded complexes $(K_1, \vx_1)$ and $(K_2, \vx_2)$ are \emph{isomorphic}, written $(K_1,\vx_1) \cong (K_2,\vx_2)$,  if there is a simplicial isomorphism $\phi\colon K_1 \to K_2$ such that $\vx_2 \circ \phi = \vx_1$ on vertices; this coordinate-preserving relation retains the triangulation. 
Note that SWL, which ignores coordinates, cannot distinguish $(K, \vx)$ from $(K, \vx')$ for any two embeddings of the same abstract complex.
It is also strictly finer than equality of the geometric supports $|K_1,\mathbf{x}_1| = |K_2,\mathbf{x}_2|$, the relation characterized by the ECT.
 
\subsection{The Euler characteristic transform}
 
For an embedded complex $(K, \vx)$, direction $\nu \in \Sd$, and threshold $t \in \R$, define the sublevel set $K_{\nu,t} := \{\sigma \in K : t_\nu(\sigma) \leq t\}$, where
$
t_\nu(\sigma) := \max\{\vx_u \cdot \nu : u \in \vertK(\sigma)\}
$
is the entry time of $\sigma$ in the direction-$\nu$ filtration. 
The \emph{Euler characteristic transform} is
\begin{equation}
\ECT(K, \vx)(\nu, t) = \chi(K_{\nu,t}) = \sum_{\sigma \in K} (-1)^{\dim \sigma} \, \1\{t_\nu(\sigma) \leq t\}.
\end{equation}
Since $\tau\subseteq\sigma$ implies $t_\nu(\tau)\le t_\nu(\sigma)$ (the maximum is taken over a subset of vertices), $K_{\nu,t}$ is closed under faces and is therefore a subcomplex of $K$.
The ECT is a \emph{complete} invariant:
\begin{theorem}[\citealt{turner2014persistent}, \citealt{ghrist2018persistent}, \citealt{curry2022many}]
For compact simplicial complexes, $\ECT(|K, x|) = \ECT(|K', x'|)$ if and only if $|K, x| = |K', x'|$ as subsets of $\R^d$, where $|K, x|$ is the compact geometric support.
\end{theorem}
 
\section{Geometric simplicial Weisfeiler--Lehman test}\label{sec:gswl}

We modify SWL by initializing vertices with their coordinates instead of a uniform color, while leaving the adjacency structure (boundary and coboundary) unchanged. 
After a sufficient number of refinement steps, each simplex recovers the coordinates of all its vertices.

\begin{definition}[GSWL]
\label{def:gswl}
Fix an embedded simplicial complex $(K,\vx)$ with injective $\vx$. Define
colors recursively by
\[
c^{(0)}_\sigma :=
\begin{cases}
(0,\,\vx_v) & \text{if } \sigma=\{v\}\text{ is a vertex,}\\
(k,\,\Phi_k(\vx_{v_0},\dots,\vx_{v_k})) & \text{if } \sigma=\{v_0,\dots,v_k\}\text{ with }k\ge 1,
\end{cases}
\]
where $\Phi_k$ is a fixed permutation-invariant function of the vertex coordinates of $\sigma$ (e.g., sorted tuple of coordinates, or the empty feature if only dimension is used).  
For $\ell\ge 0$,
\[
c^{(\ell+1)}_\sigma := \mathrm{HASH}\!\left(
  c^{(\ell)}_\sigma,\;
  \{\!\{c^{(\ell)}_\tau : \tau\in\partial\sigma\}\!\},\;
  \{\!\{c^{(\ell)}_\rho : \rho\in\coface(\sigma)\}\!\}
\right),
\]
where $\mathrm{HASH}$ is injective.  
The feature maps $\Phi_k$ and the injective hash are shared across all complexes being compared; colors are understood as exact symbolic labels rather than finite-precision hash values. 
We call two embedded complexes \emph{GSWL-$L$ equivalent}, denoted by $K_1\equiv_L K_2$, if $\{\!\{c^{(L)}_\sigma:\sigma\in K_1\}\!\} =\{\!\{c^{(L)}_\sigma:\sigma\in K_2\}\!\}$.
\end{definition}

\begin{wrapfigure}{r}{0.48\textwidth}
    \vspace{-\intextsep}
    \centering
    \includegraphics[width=0.47\textwidth]{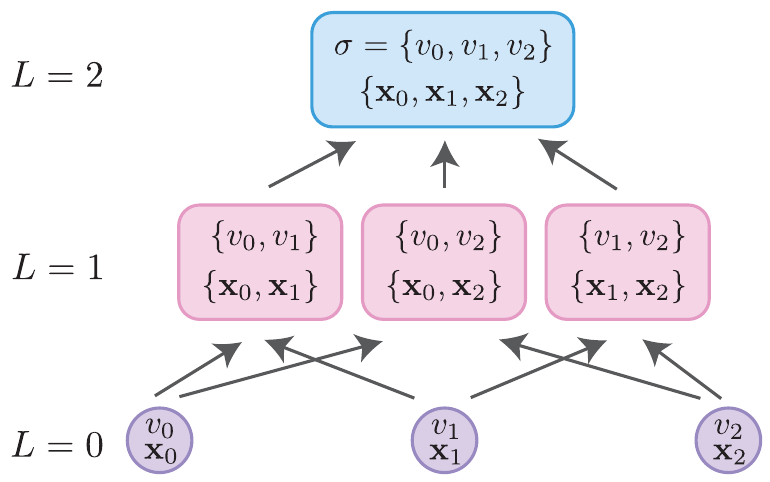}
    \caption{
        Coordinate recovery in GSWL. We need $k$ propagation steps to
        recover \mbox{$k$-simplices}.
    }
    \label{fig:coord_recovery}
    \vspace{-\intextsep}
\end{wrapfigure}

The key difference to SWL~\citep{bodnar2021weisfeiler} is the initialization: Each simplex receives features derived from its vertex coordinates rather than a uniform color.  
In particular, vertices carry their coordinates, and higher-dimensional simplices carry $\Phi_k$, a \emph{permutation-invariant} function of their vertex coordinates such as edge midpoint and length, or triangle centroid and area, as used in our experiments (Section~\ref{sec:experiments}).  
Since this initialization is strictly finer than SWL's constant coloring, GSWL is \emph{at least as powerful} as SWL on abstract complexes and \emph{strictly more powerful} on embedded complexes.  
The choice $\Phi_k\equiv\emptyset$ (dimension-only initialization for non-vertex simplices) is the minimal instance of Definition~\ref{def:gswl} that still suffices for Lemma~\ref{lem:coord-recovery} and all downstream results; richer choices of $\Phi_k$ only refine the coloring further.
Notice that for a vertex $\sigma=\{v\}$ the boundary $\partial\sigma$ is empty, so $\mathrm{HASH}$ receives the empty multiset in the boundary slot; analogously for top-dimensional simplices in the coboundary slot.  
We adopt the convention that $\mathrm{HASH}$ is well-defined \emph{on} and injective \emph{in} inputs that include empty multisets.
Figure~\ref{fig:coord_recovery} illustrates the induction in Lemma~\ref{lem:coord-recovery}: vertex coordinates propagate upward through the Hasse diagram from vertices to edges and then to triangles.
 
\begin{restatable}[Coordinate recovery]{lemma}{coordinaterecovery}
\label{lem:coord-recovery}
Let $(K, \vx)$ be an embedded simplicial complex whose vertex map
$\vx\colon\vertK(K)\to\mathbb{R}^d$ is injective, and let
$\sigma=\{v_0,\dots,v_k\}$ be a $k$-simplex of $K$.  If $L\ge k$, then
$c^{(L)}_\sigma$ determines the unordered set
$\{\vx_{v_0},\dots,\vx_{v_k}\}$.  
In particular, the entry time
$t_\nu(\sigma)=\max_{u\in\vertK(\sigma)} \vx_u\cdot \nu$ is determined
for every $\nu\in S^{d-1}$.
\end{restatable}

\begin{restatable}{corollary}{completeness}
\label{cor:completeness}
    Let $\vx_1, \vx_2$ be injective and $L \geq \max\{\dim K_1, \dim K_2\}$. Then $K_1 \equiv_L K_2$ if and only if $(K_1,\mathbf{x}_1) \cong (K_2,\mathbf{x}_2)$.
\end{restatable}

The implication of Lemma~\ref{lem:coord-recovery} is that, given a sufficiently deep network, we are guaranteed to recover coordinates; as our experiments show~(see Section~\ref{sec:Cobundary messages}), deeper networks exhibit monotone performance improvements. Lemma~\ref{lem:coord-recovery} thus provides the precise computational mechanism to combine message passing with the ECT.

\begin{figure}[b]
    \centering
    \includegraphics[width=0.725\textwidth]{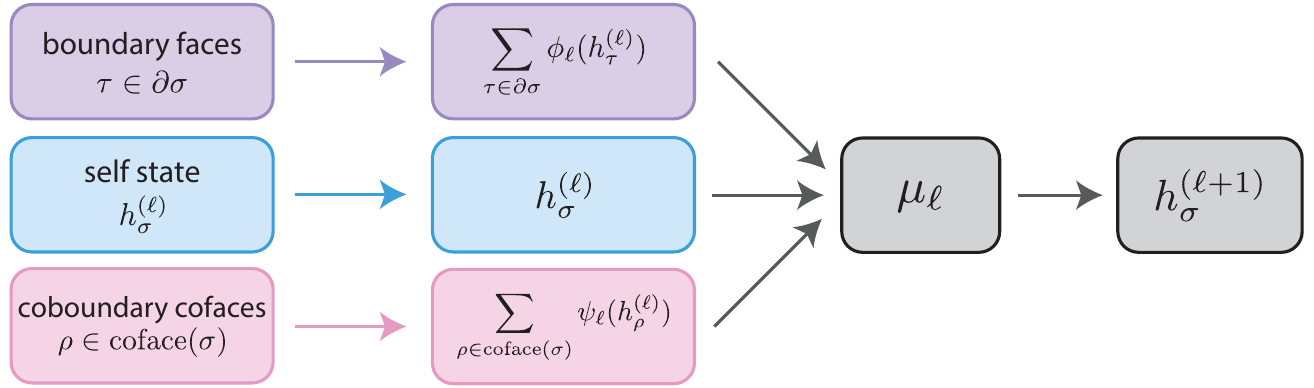}
    \caption{%
    Simplicial message passing update for a target simplex $\sigma$. 
Boundary and coboundary messages are aggregated separately and combined with the current state of $\sigma$.
    }
    \label{fig:architecture}
\end{figure}
 
\section{A geometry-aware simplicial architecture and its expressivity}
\label{sec:theory}

We now introduce a message-passing architecture matching the boundary/coboundary structure of GSWL. 
We prove that it cannot exceed GSWL in expressivity, but can match GSWL on any fixed finite family. 
Each simplex is updated by
\begin{equation}\label{eq:presum}
h_\sigma^{(\ell+1)} = \mu_\ell\!\left(h_\sigma^{(\ell)},\; \sum_{\tau \in \partial \sigma} \phi_\ell(h_\tau^{(\ell)}),\; \sum_{\rho \in \coface(\sigma)} \psi_\ell(h_\rho^{(\ell)})\right),
\end{equation}
where $\tau$ ranges over the boundary faces of $\sigma$, $\rho$ ranges over the cofaces of $\sigma$, and $\mu_\ell$, $\phi_\ell$, $\psi_\ell$ are \emph{learnable maps}, with $\phi_\ell$ and $\psi_\ell$ applied before summation.
Figure~\ref{fig:architecture} illustrates this update where it separately aggregates messages from boundary faces and coboundary cofaces before combining them with the current state of $\sigma$.
We initialize this  architecture with $h_\sigma^{(0)} = E(c_\sigma^{(0)})$ for a deterministic encoder $E$.
For instance, we can set $E(\dim\sigma,f) \coloneqq W_{\dim\sigma}\,f+b_{\dim\sigma}$, i.e., a  per-dimension linear map  that acts on the coordinate-derived features introduced in Section~\ref{sec:background}~(vertex coordinates, edge midpoint and length, triangle centroid and area).
We obtain the following \emph{upper bound} for the expressivity of the architecture.
\begin{restatable}[Upper bound]{theorem}{upperbound}
\label{thm:upper}
Fix $L\ge 0$.  Suppose $K_1\equiv_L K_2$ under GSWL.  Then for every
choice of parameters in the architecture~\eqref{eq:presum} and every
readout of the form
$
z(K) \;=\; \sum_{\sigma\in K} \eta_{\dim\sigma}\, \Psi\!\left(h^{(L)}_\sigma\right)
$,
where $\eta_k$ are fixed per-dimension weights and $\Psi$ is any
deterministic map, we have $z(K_1)=z(K_2)$.
More generally, the same conclusion holds for any deterministic
readout that depends only on the multiset
$\{\!\{(\dim\sigma,h_\sigma^{(L)}):\sigma\in K\}\!\}$.
\end{restatable}
Conversely, the architecture matches GSWL on finite families of simplicial complexes.
\begin{definition}[Finite interpolation property]\label{def:fip}
A function class $\mathcal{M}\subseteq\{f:\R^p\to\R^q\}$ has the \emph{finite interpolation property} if, for every finite set of input/output pairs $\{(\xi_i,y_i)\}_{i=1}^n\subset\R^p\times\R^q$ with the $\xi_i$ pairwise distinct, there exists $f\in\mathcal{M}$ with $f(\xi_i)=y_i$ for all $i$.
Standard MLP classes (feedforward networks with ReLU activation and at least one hidden layer of sufficient width) satisfy this property; cf.\ citet{xu2019powerful}.
\end{definition}

For the experiments, each $\Phi_k$ takes values in a finite-dimensional Euclidean space. 
We represent an initial color by concatenating its dimension tag with its features and padding with
zeros to a common input dimension. Thus distinct observed initial
colors correspond to distinct points in some $\R^p$
 
\begin{restatable}[Finite-family realizability]{theorem}{realizability}
\label{thm:lower}
Let $\mathcal{F}$ be a finite family of finite embedded simplicial complexes and fix $L\ge 0$.  For $\ell\in\{0,\dots,L\}$, let
$
\mathscr{C}_\ell \;:=\; \{c^{(\ell)}_\sigma : \sigma\in K,\ K\in\mathcal{F}\}$, $m_\ell := |\mathscr{C}_\ell|$, and
$m := \max_{0\le\ell\le L} m_\ell$.
Assume the function class used for $E,\phi_\ell,\psi_\ell,\mu_\ell,\Psi$ satisfies the finite interpolation property of Definition~\ref{def:fip}.
Then, for any hidden dimension $H \ge 3m$, there exist parameters in the architecture~\eqref{eq:presum} and a readout $z(K)=\sum_\sigma \Psi(h^{(L)}_\sigma)$ such that $z(K_1)=z(K_2)\iff K_1\equiv_L K_2$ for all $K_1,K_2\in\mathcal{F}$.
\end{restatable}

\begin{proof}[Proof sketch]
Fix three disjoint coordinate blocks $A,B,C\subset\{1,\dots,H\}$ of size $m$ each, and encode each GSWL color at round $\ell$ as a distinct basis vector $a_j\in A$.  
Choose $\phi_\ell,\psi_\ell$ to move $A$-supported basis vectors to $B$ and $C$; boundary and coboundary sums then realize the multiplicity vectors of the color multisets in blocks $B$ and $C$.  
Choose $\mu_\ell$ to read the resulting triple and write a fresh basis vector back into $A$, overwriting the previous layer's state. 
After $L$ rounds, states are injective in observed GSWL colors; the readout $\Psi$ mapping $a_j\mapsto e_j$ produces the round-$L$ color histogram, whose equality on $K_1,K_2$ is exactly $K_1\equiv_L K_2$.  
See Appendix~\ref{app:proofs} for the full construction.
\end{proof}

Notice that the bound $H \ge 3m$ uses three disjoint coordinate blocks that are reused across layers.
The dependence of $H$ on $\mathcal{F}$ is inherent to basis-encoding realizability arguments of this kind, as in the standard WL/GIN construction of \citet{xu2019powerful}.
For the ECT connection in Section~\ref{sec:ect}, the key consequence is simplex-wise injectivity, i.e., $h^{(L)}_\sigma$ is injective in the round-$L$ GSWL color of $\sigma$. 
Together with Lemma~\ref{lem:coord-recovery}, this means that for $L\ge\dim\sigma$, the hidden state of $\sigma$ determines its entry time $t_\nu(\sigma)$ in every direction $\nu$, so a simplex-wise readout can recover the indicator contributions whose signed sum gives sampled ECT values.
 
\section{Connection to the Euler characteristic transform}\label{sec:ect}

The bounds in Section~\ref{sec:theory} characterize expressivity in terms of GSWL, whose colors are just hash values. 
To give the characterization geometric content, we connect GSWL to the Euler Characteristic Transform.
It is a complete invariant on the geometric support and hence coarser than GSWL.
In this section, we combine the finite-family lower bound (Theorem~\ref{thm:lower}) with the coordinate recovery results~(Lemma~\ref{lem:coord-recovery}) to show that our proposed architecture computes sampled ECT values exactly on finite families.
Additionally, ECT stability controls errors in recovered coordinates.

\subsection{Exact realization on finite families}

\begin{restatable}[Exact sampled ECT on finite families]{theorem}{exactectfinite}
\label{thm:ect}
Let $\mathcal{F}$ be a finite family of embedded simplicial complexes $(K,\vx)$ of dimension at most $D$, fix finite sets $\mathcal{V}\subset S^{d-1}$ and $T\subset\R$, and let $L\ge D$.  
There exist parameters in the architecture~\eqref{eq:presum} and a simplex-wise map $\Psi:\R^H\to\R^{|\mathcal{V}|\cdot|T|}$ such that, for every $(K,\vx)\in\mathcal F$,
\begin{equation}
\sum_{\sigma\in K}(-1)^{\dim\sigma}\Psi(h^{(L)}_\sigma)=
\bigl(\ECT(K,\vx)(\nu,t)\bigr)_{(\nu,t)\in\mathcal V\times T}.
\end{equation}
\end{restatable}

The theorem holds for arbitrary finite choices of directions $\mathcal{V}$ and thresholds $T$, and thus recovers the ECT exactly at the queried grid points. 
Since each directional slice $t\mapsto \ECT(K,\vx)(\nu,t)$ is a piecewise-constant function with jumps only at the finitely many entry times ${t_\nu(\sigma): \sigma \in K}$, a grid $T$ containing all such values across $\mathcal{F}$ suffices to reconstruct the full step function for each direction. 
The depth requirement $L \ge D$ matches the dimension of the complex; for example, $L=2$ suffices for triangle meshes. 
More generally, increasing depth expands the portion of the Hasse diagram from which coordinate information can be aggregated, which is consistent with the monotonic improvements observed in our depth ablation experiments (Section~\ref{sec:experiments}).

\subsection{Approximation on bounded classes via ECT stability}

Theorem~\ref{thm:ect} provides exact computation on finite families. 
To extend to \emph{infinite} classes, we rely on a stability result for the ECT under perturbations of the embedding.
\citet{george2025stability} define a distance between the ECTs of two embeddings $\vx, \vx'$ of the same abstract complex $K$, i.e.,
\begin{equation}
d_{\ECT}\bigl(\ECT(K, \vx),\, \ECT(K, \vx')\bigr)
:= \int_{\nu \in S^{d-1}}
\|\mathrm{ECC}_\nu(K, \vx) - \mathrm{ECC}_\nu(K, \vx')\|_1 \, d\nu,
\end{equation}
and prove the following bound.

\begin{theorem}[{\citealt[Theorem~4.1]{george2025stability}}]
\label{thm:stability}
Let $K$ be a finite abstract simplicial complex and let $\vx, \vx': \vertK(K) \to \mathbb{R}^d$ be two embeddings. Then
\[
d_{\ECT}\bigl(\ECT(K, \vx),\, \ECT(K, \vx')\bigr)
\;\leq\; 2\, C_K\, C_d \sum_{v \in \vertK(K)} \|\vx(v) - \vx'(v)\|_2,
\]
where $C_K$ depends on the combinatorial structure of $K$ and $C_d$ depends on the ambient dimension $d$.
\end{theorem}

We next consider a class of embeddings of a fixed finite abstract simplicial complex and state a conditional error bound.

\begin{assumption}[Bounded embeddings]\label{ass:bounded}
    Fix a finite abstract simplicial complex $K$ and a class $\mathcal{C}$ of embeddings $\vx: \vertK(K)\to\R^d$ such that every $\vx \in \mathcal{C}$ has image in a common compact set $B\subset\R^d$.
\end{assumption}

Theorem~\ref{thm:lower} provides injective hidden states on the colors observed in a fixed finite family.
Its interpolating neural maps can be continuous, but the construction gives no uniform coordinate-recovery guarantee outside that family. 
The following result therefore assumes entry-time recovery and bounds the resulting ECT error; it does not assert that the architecture satisfies this assumption on every bounded class.

\begin{restatable}[Coordinate-preserving representations]{corollary}{ectstability}
    \label{cor:stability}
    Under Assumption~\ref{ass:bounded}, suppose a representation determines measurable entry-time estimates satisfying $\sup_{\sigma,\nu}
  |\widehat t_\nu(\sigma)-t_\nu(\sigma)|\leq\delta$ uniformly over $\mathcal{C}$.
  Let $\widehat{\operatorname{ECT}}$ be obtained by substituting these estimates into Eq.~(3). 
  Then $d_{\mathrm{ECT}}\bigl(
    \operatorname{ECT}(K,x),\widehat{\operatorname{ECT}}
  \bigr)\leq |K|\delta$.
  Coordinate recovery within $\varepsilon$ gives the same bound with $\delta=\varepsilon$.
\end{restatable} 

This corollary assumes every-time recover rather than establishing it, unlike Theorem~\ref{thm:lower}.
Under minimal initialization, a $k$-simplex's state is embedding-independent for $L<k$; without boundary messages, all positive-dimensional states remain so at every depth.

\section{Experiments}\label{sec:experiments}

We evaluate the theoretical predictions from Sections~\ref{sec:gswl}--\ref{sec:ect} on three datasets of increasing complexity: synthetic triangulations with controlled deformations, manifold triangulations~\citep{mantra2025} with synthetic embeddings, and registered human body meshes~\citep{bogo2014faust}. 
Five architectures are compared: Combinatorial SMP (constant initial features, neural analogue of SWL), Geometry-Aware SMP (vertex coordinates with coordinate-derived edge and triangle features), GCN and GIN on the 1-skeleton with vertex coordinates~\citep{kipf2017semi,xu2019powerful}, and a DeepSets~\citep{zaheer2017deepsets} baseline (permutation-invariant MLP on vertex coordinates, no topology).
Appendix~\ref{app:exp_details} provides implementation details.
The trained models use fixed width, learned maps, LayerNorm residual connections, and mean pooling rather than the constructions of Section~\ref{sec:theory}.
Notice that such residual connections are bounded \emph{above} by GSWL but not proven to attain it, so the experiments illustrate the predicted hierarchy rather than the bounds.

\subsection{Geometric information}
\label{sec:Geometric Information}

We first verify the lower bound prediction: On datasets where all samples share the \emph{same} abstract complex, combinatorial models~\citep{bodnar2021weisfeiler, bodnar21weisfeilerb} are unable to learn and yield \emph{identical} outputs.
To demonstrate this lack of expressivity, we first devise an experiment for classifying \emph{deformations}.
We deform a Delaunay triangulation ($V\!=\!40$, $E\!=\!107$, $T\!=\!68$) via four smooth map families~(bend, twist, stretch, random smooth), producing 300 samples for a balanced 4-class classification problem (chance $= 0.25$).  All samples share identical connectivity. 
As Table~\ref{stab:deformation-classification} shows, combinatorial simplicial message passing remains at essentially random guessing , confirming that the model produces \emph{identical} representations across all inputs, so optimization cannot distinguish between classes. 
By contrast, our geometry-aware simplicial message passing scheme is capable of recovering all deformations; zeroing all coordinate input in this architecture yields random-level performance, as expected.
Following this, to connect directly to Theorem~\ref{thm:ect}, we regress sampled ECT vectors (8 directions $\times$ 10 thresholds $= 80$ dimensions) from vertex coordinates on the same triangulation. 
Table~\ref{stab:synthetic-meshes} reports test MSE. 
Our geometry-aware model achieves a more than $2\times$ reduction over the combinatorial baseline. 
A depth ablation shows monotonic improvement from $L\!=\!1$ (MSE $0.121$) to $L\!=\!6$ (MSE $0.099$), consistent with the coordinate recovery lemma, i.e., deeper networks access finer ECT grids.

\begin{table}[btp]
\centering
\caption{%
  Experiments on synthetic datasets. (a)~Deformation classification (random chance $= 0.25$) using a shared abstract complex. (b)~Regression of ECT vectors on synthetic meshes.
}
\label{tab:synthetic-experiments}
\small
\subcaptionbox{\label{stab:deformation-classification}}{
\begin{tabular}{l S[table-format=1.3 \pm 1.3, separate-uncertainty=true]}
\toprule
Model & {Test Accuracy~($\uparrow$)} \\
\midrule
MLP (flattened coordinates) & \bfseries 0.976 +- 0.016 \\
Combinatorial SMP & 0.224 +- 0.013 \\
Geometry-Aware SMP (ours) & \bfseries 0.976 +- 0.023 \\
\bottomrule
\end{tabular}
}
\subcaptionbox{\label{stab:synthetic-meshes}}{%
\begin{tabular}{lS[table-format=1.3 \pm 1.3, separate-uncertainty=true]}
\toprule
Model & {Test MSE~($\downarrow$)}\\
\midrule
MLP (flattened coordinates) & 0.174 +- 0.063 \\
Combinatorial SMP & 0.367 +- 0.108 \\
Geometry-Aware SMP (ours) & \bfseries 0.157 +- 0.062\\
\bottomrule
\end{tabular}
}
\end{table}

\subsection{Using higher-order simplices}

We next ask whether the full simplicial adjacency structure (including triangles)
provides value \emph{beyond} graph-level message passing or unstructured point cloud
processing.
To this end, we use the FAUST dataset, which consists of 100 registered meshes of 10 human subjects in 10 poses, all sharing a template with $V\!=\!6890$ vertices. 
We perform 10-way pose classification with 5-fold cross-validation and $5\times$ augmentation (small rotations and Gaussian noise; see Appendix~\ref{app:exp_details}). Table~\ref{tab:faust} reports mean accuracy $\pm$ standard deviation across folds.
The results exhibit a hierarchy. 
Combinatorial SMP collapses to a constant per-dimension output across all inputs~(all meshes in FAUST share the same abstract complex), so the network is stuck at the majority-class baseline across folds, yielding the degenerate 0.062 accuracy with zero variance reported in Table~\ref{tab:faust}. 
This is the expected consequence of Theorem~\ref{thm:upper} when applied to a dataset of constant abstract structure: The model produces the same representation for every input, so only class-frequency information is available to the readout and the network cannot break symmetry between classes. 
DeepSets and GCN perform comparably~($\sim\!0.74$), indicating that edge connectivity alone adds little information over unstructured coordinate access for this task. 
GIN, with its more expressive aggregation, reaches $0.80\pm 0.04$. 
Finally, geometry-aware SMP achieves $0.85\pm 0.03$, a small improvement whose error bars overlap with GIN; we therefore read the FAUST results as establishing a clear trend from combinatorial models through point-cloud and graph baselines to simplicial message passing, without claiming a statistically separated gap between the top two models~(which would require larger-scale evaluations beyond the scope of our present work).  
We note that specialized shape analysis methods achieve near-perfect accuracy on FAUST using task-specific geometric descriptors~\citep{bogo2014faust}; the present comparison is intended to isolate the contribution of each structural level, \emph{not} to compete on the benchmark.

\begin{table}[btp]
\centering
\caption{%
    FAUST 10-way pose classification (chance $= 0.10$). Our geometry-aware architecture outperforms existing simplicial message-passing schemes.
}
\label{tab:faust}
\small
\begin{tabular}{l l S[table-format=1.3(3)]}
\toprule
Model & Structure & {Accuracy} \\
\midrule
DeepSets                    & vertex coordinates            & 0.75 +- 0.04 \\
GCN + coords               & vertices + edges              & 0.74 +- 0.07 \\
GIN + coords               & vertices + edges              & 0.80 +- 0.04 \\
Combinatorial SMP          & abstract complex              & 0.06 +- 0.00 \\
Geometry-Aware SMP (ours)  & vertices + edges + triangles  & \bfseries 0.85 +- 0.03 \\
\bottomrule
\end{tabular}
\end{table}

The previous experiment tests discrimination between coarse categories. 
We now test whether the same hierarchy of models appears on a \emph{regression} target.
To this end, in compound deformations of the synthetic triangulation, we regress a $30$-dimensional vector of geometric summary statistics, i.e., 
\begin{inparaenum}[(i)]
    \item vertex displacement norms~(moments of $\|\vx_v-\vx_v^{\text{base}}\|$),
    \item edge-length distributions~(moments and percentiles of $\{\|\vx_u-\vx_v\|:\{u,v\}\in E\}$),
    \item per-triangle area statistics, and
    \item angle defects ($K(v)=2\pi-\sum_{\text{incident}}\theta_i$). 
\end{inparaenum}
Our geometry-aware SMP model achieves MSE $0.349$ compared to $0.643$ for DeepSets and $1.104$ for Combinatorial SMP. 
The gap between SMP and DeepSets reflects the target's dependence on connectivity-derived quantities (e.g., angle defects, area distributions) that are \emph{inaccessible} to a point cloud model.
The gap from DeepSets to a geometry-aware  model shows its dependence on coordinates.

\subsection{Coboundary messages are necessary}
\label{sec:Cobundary messages}

\begin{wraptable}{r}{0.5\textwidth}
\centering
\vspace{-\intextsep}
\caption{Per-vertex curvature prediction (test MSE, lower is better) \emph{with} and \emph{without} coboundary messages for different numbers of message-passing layers.}%
\label{tab:cobdy}%
\setlength{\tabcolsep}{2pt}%
\resizebox{\linewidth}{!}{%
\begin{tabular}{lccc}
\toprule
Depth $L$ & Full & Boundary-only & Gap\\
\midrule
1 & $\mathbf{0.345}$\,{\scriptsize$\pm$0.110} & 0.671\,{\scriptsize$\pm$0.126} & $+0.326$\,{\scriptsize$\pm$0.075} \\
4 & $\mathbf{0.203}$\,{\scriptsize$\pm$0.092} & 0.600\,{\scriptsize$\pm$0.121} & $+0.397$\,{\scriptsize$\pm$0.061} \\
8 & $\mathbf{0.181}$\,{\scriptsize$\pm$0.086} & 0.561\,{\scriptsize$\pm$0.123} & $+0.381$\,{\scriptsize$\pm$0.057} \\
\bottomrule
\end{tabular}
}
\vspace{-\intextsep}
\end{wraptable}
We verify that the coboundary channel (higher-dimensional simplices passing
information to lower-dimensional ones)  \emph{crucially} contributes to geometric expressivity.
To this end, we aim to perform per-vertex curvature prediction on a mesh.
Our target is discrete Gaussian curvature $K(v) = 2\pi - \sum_{\text{incident}} \theta_i$, which depends on angles in incident triangles. 
Predicting this quantity at each vertex requires triangle-level information to flow downward through coboundary aggregation. 
Table~\ref{tab:cobdy} compares the full architecture (boundary + coboundary) against a boundary-only ablation across different network depths.
Both models improve with depth, but
the full model \emph{consistently outperforms} the boundary-only variant with reported gaps of $0.326$--$0.397$. 
On the simpler task of predicting the number of incident triangles per vertex, the full model achieves MSE $< 10^{-4}$ while a boundary-only model merely reaches an MSE $0.68$.

\subsection{Sanity checks and generalization}

%
Having discussed the necessity of all components, we perform several ``sanity checks'' and generalization experiments.
Concerning \emph{permutation equivariance},
under random vertex relabeling on the synthetic classification task, an MLP baseline on flattened coordinates drops from $0.984$ to $0.276$ test accuracy~(indicating memorization ), while our geometry-aware SMP outputs are exactly invariant.
Next, we show that our architecture does not introduce any spurious signals by classifying the \emph{orientability} of a simplicial complexes from the MANTRA dataset~\citep{mantra2025}, using $400$ balanced samples, each with a \emph{distinct} abstract complex.
This dataset is notable for its absence of geometrical information; and indeed, we observe that both combinatorial and geometry-aware SMP achieve an accuracy of about $0.62$, with performance above chance being attributable to combinatorial correlates (e.g., Euler characteristic or simplex counts). A geometry-aware model gains \emph{no} advantage, confirming that it does not introduce spurious geometric signals.
Using the deformation classification task~(cf.\ Section~\ref{sec:Geometric Information}) and \emph{embedding} triangulations~(with 6--12 vertices) into $\R^2$ using a spectral layout, we observe that our geometry-aware model achieves near-perfect accuracy~(Table~\ref{tab:mantra}).
%
\begin{wraptable}{r}{0.475\textwidth}
\vspace{-\intextsep}
\centering
\caption{Deformation classification on MANTRA triangulations embedded into $\R^2$~(chance $= 0.25$). We see consistent gains of our geometry-aware message-passing scheme.
}
\label{tab:mantra}
\small
\setlength{\tabcolsep}{4pt}
\begin{tabular}{lccc}
\toprule
 & DeepSets & Comb.\ SMP & Ours \\
\midrule
S$^2$  & 0.49\,{\scriptsize$\pm$0.03} & 0.21\,{\scriptsize$\pm$0.02} & $\mathbf{0.91}$\,{\scriptsize$\pm$0.08} \\
T$^2$  & 0.54\,{\scriptsize$\pm$0.08} & 0.20\,{\scriptsize$\pm$0.02} & $\mathbf{1.00}$\,{\scriptsize$\pm$0.00} \\
K$^2$  & 0.53\,{\scriptsize$\pm$0.08} & 0.20\,{\scriptsize$\pm$0.02} & $\mathbf{0.96}$\,{\scriptsize$\pm$0.04} \\
RP$^2$ & 0.73\,{\scriptsize$\pm$0.05} & 0.20\,{\scriptsize$\pm$0.02} & $\mathbf{0.98}$\,{\scriptsize$\pm$0.02} \\
\bottomrule
\end{tabular}
\end{wraptable}
%
Combinatorial SMP and DeepSets, by contrast, remain at random-level performance or vary with the topology of the underlying surface.
The consistent performance of our method across different topological types confirms that our advantage is \emph{not} an artifact of any particular triangulation.
Finally, on \emph{rotation-invariant regression targets}~(mesh extent, mean centroid distance, distance variance) with $5$-fold CV, our model achieves MSE $0.036 \pm 0.010$, compared to combinatorial SMP at $1.016 \pm 0.178$,  a $28\times$ reduction. 
DeepSets achieves $0.048 \pm 0.011$, matching our model, which is expected since the targets are permutation-invariant functions of vertex coordinates that do not depend on connectivity, so a geometrical-topological inductive bias cannot provide additional benefits.

\subsection{Limitations}
\label{sec:limitations}

We identify some limitations in our experiments.
First, we find that all synthetic experiments saturate at nigh-perfect accuracy or near-zero MSE, which is partially due to the dataset size.
Thus, these experiments prove the existence of an expressivity gap but do not constitute a stress test of our architecture on highly-challenging tasks.
Our experiments on FAUST provide a more meaningful benchmark, even though the dataset size limits the statistical power.
%
Theorem~\ref{thm:ect} guarantees exact values on a prescribed finite grid, not completeness of an \emph{arbitrary} grid.
A fixed finite family admits a family-dependent grid separating distinct geometric supports;
GSWL completeness itself does not depend on ECT sampling.
Finally, additional experiments could also incorporate other point cloud classification models~\citep{qi2017pointnet, wang2019dynamic}. However, our main focus  lies on \emph{characterizing expressivity}, and we find that
our baselines already help in isolating the contribution of higher-order simplices, and support our claims.

\section{Conclusion}

We introduce the \emph{Geometric Simplicial Weisfeiler--Lehman}~(GSWL) test, a color refinement procedure for embedded simplicial complexes, and establish its relationship to geometry-aware simplicial message-passing architectures.
Our proposed architecture is bounded above by GSWL~(Theorem~\ref{thm:upper}), and matches GSWL exactly on finite families~(Theorem~\ref{thm:lower}). 
GSWL is complete for coordinate-preserving simplicial isomorphism once the depth reaches the maximal simplex dimension.
Combining this  with simplex-wise coordinate recovery yields exact computation of sampled ECT values~(Theorem~\ref{thm:ect}).
Separately, a recent result on ECT stability~\citep{george2025stability} provides $\varepsilon$-approximation of the full ECT on bounded classes~(Theorem~\ref{cor:stability}). 
Our results provide a \emph{geometric expressivity characterization} analogous to the WL--GNN equivalence for graphs.
We illustrate our results experimentally, observing that
on datasets with shared abstract connectivity, combinatorial models perform provably worse, whereas our geometry-aware model succeeds. 
Coboundary ablations confirm that higher-dimensional information flow is necessary for tasks where geometry matters.
As for future work, it would be interesting to extend the framework to function-valued ECT recovery, i.e., outputting $\ECT(K)(v,t)$ for an arbitrary query $(v,t)$) and higher-dimensional complexes.

\ifarXiv
    \subsection*{Acknowledgments}
    This work has received funding from the Swiss State Secretariat for
    Education, Research, and Innovation~(SERI).
    The authors declare no competing interests. The funders had no role in
    the preparation of the manuscript or the decision to publish.
\fi

\subsection*{AI use statement}

In this work, we have not used generative AI tools for any tasks with required
disclosure. Additionally, we used generative AI tools for polishing parts of the
writing and editing code. We have reviewed all such AI-assisted work.
Specifically, we have further reviewed all code produced by an LLM thoroughly.
We take responsibility for the final content of this work, including text,
claims, and artifacts produced with the aid of generative AI.

\subsection*{Reproducibility statement}

Sections~\ref{sec:gswl}--\ref{sec:ect} specify the GSWL procedure, message-passing architecture, and theoretical assumptions, with proofs provided in the Appendix~\ref{app:proofs}.
Section~\ref{sec:experiments} and Appendix~\ref{app:exp_details} describe the datasets, evaluation protocols, preprocessing, model configurations, hyperparameters, and computational resources used in our experiments.
We also provide the complete source code for all experiments in the supplementary material.
 
\bibliographystyle{abbrvnat}
\bibliography{references}

\clearpage
\appendix
 
\section{Full proofs}\label{app:proofs}

\coordinaterecovery*

\begin{proof}
We prove this by induction on $k=\dim\sigma$, with $L\ge k$ as a parameter.

\emph{Base case ($k=0$).}  By injectivity of $\mathrm{HASH}$, $c^{(L)}_v$ determines $c^{(L-1)}_v$ and, iterating, $c^{(0)}_v=(0,\vx_v)$, hence it also determines $\vx_v$.

\emph{Inductive step.}  Suppose the claim holds for all $(k{-}1)$-simplices and all $L'\ge k-1$, and let $\sigma=\{v_0,\dots,v_k\}$ with $L\ge k$.  
Since $\mathrm{HASH}$ is injective, $c^{(L)}_\sigma$ determines the multiset $\{\!\{c^{(L-1)}_\tau:\tau\in\partial\sigma\}\!\}$.  
Each facet $\tau=\sigma\setminus\{v_i\}$ is a $(k{-}1)$-simplex and $L-1\ge k-1$, so by the inductive hypothesis $c^{(L-1)}_\tau$ determines $\{\vx_{v_0},\dots,\widehat{\vx_{v_i}},\dots,\vx_{v_k}\}$.  
Because $\vx$ is injective on $\vertK(K)$, the union of these $k+1$ facet vertex-coordinate sets equals $\{\vx_{v_0},\dots,\vx_{v_k}\}$ unambiguously.
\end{proof}

\completeness*
\begin{proof}
Equality of the global round-$L$ color multisets (simplex dimension is recoverable from any round's color, since the initialization encodes $\dim\sigma$) gives a bijection of simplices preserving dimension and, by Lemma~\ref{lem:coord-recovery}, the unordered vertex-coordinate set. Singleton sets and injectivity of the vertex maps yield a coordinate-preserving vertex bijection; each higher-dimensional coordinate set then determines exactly which vertex subset spans a simplex. Conversely, an isomorphism preserves the initialization and every boundary and coboundary multiset, hence all colors.
\end{proof}

\upperbound*

\begin{proof}
By induction on $\ell$ there exists a deterministic map $F_\ell$ from $\ell$-round GSWL colors to hidden states such that $h^{(\ell)}_\sigma=F_\ell(c^{(\ell)}_\sigma)$.  
The base case is $F_0=E$.  
For the inductive step, injectivity of $\mathrm{HASH}$ ensures that the triple
$(c^{(\ell)}_\sigma,\{\!\{c^{(\ell)}_\tau\}\!\}_{\partial\sigma},
\{\!\{c^{(\ell)}_\rho\}\!\}_{\coface(\sigma)})$ is uniquely determined by $c^{(\ell+1)}_\sigma$.  
Because the architecture depends on neighbors only through their sums, and $F_\ell$ maps equal colors to equal hidden states, the update yields $h^{(\ell+1)}_\sigma = F_{\ell+1}(c^{(\ell+1)}_\sigma)$.

For the readout, the per-dimension weighting $\eta_{\dim\sigma}$ is compatible with the color-level argument because $\dim(\sigma)$ is the first component of $c^{(0)}_\sigma$ under Definition~\ref{def:gswl}, and is therefore determined by $c^{(\ell)}_\sigma$ for every $\ell$ (injectivity of $\mathrm{HASH}$ propagates the first component forward).  
GSWL-$L$ equivalence then implies equal multisets of $(\dim,h^{(L)})$ pairs, and so $z(K_1)=z(K_2)$.
\end{proof}

\realizability*
 
\begin{proof}
Fix bijections $\iota_\ell:\mathscr{C}_\ell\to\{1,\dots,m_\ell\}$ for each $\ell$, and choose three disjoint coordinate blocks $A,B,C\subseteq\{1,\ldots,H\}$ of size $m$ each, which is possible because $H\geq3m$.

Let $\{a_j\}_{j=1}^m$, $\{b_j\}_{j=1}^m$, $\{c_j\}_{j=1}^m$ denote the standard basis vectors of $\R^H$ supported on $A$, $B$, $C$ respectively.

\emph{Encoder.}  Define $E$ on the finite set $\mathscr{C}_0$ by $E(c)=a_{\iota_0(c)}$ and extend to $\R^p$ via Definition~\ref{def:fip}.

\emph{Inductive step.}  Assume parameters have been chosen through round $\ell$ so that $h^{(\ell)}_\sigma=a_{\iota_\ell(c^{(\ell)}_\sigma)}$ for every simplex $\sigma$ in every $K\in\mathcal{F}$.  
Define $\phi_\ell$ on $\{a_1,\dots,a_m\}$ by $\phi_\ell(a_j)=b_j$, and $\psi_\ell$ by $\psi_\ell(a_j)=c_j$, extending each by finite interpolation.  
Then
\[
\sum_{\tau\in\partial\sigma}\phi_\ell(h^{(\ell)}_\tau)
=\sum_{j=1}^m n^B_j(\sigma)\,b_j,
\qquad
\sum_{\rho\in\coface(\sigma)}\psi_\ell(h^{(\ell)}_\rho)
=\sum_{j=1}^m n^C_j(\sigma)\,c_j,
\]
where $n^B_j(\sigma)=|\{\tau\in\partial\sigma:\iota_\ell(c^{(\ell)}_\tau)=j\}|$ and $n^C_j(\sigma)$ is defined analogously for the coboundary.  
These are exactly the multiplicity vectors of the boundary and coboundary color multisets, encoded in blocks $B$ and $C$.

\emph{Update.}  The triple $\bigl(h^{(\ell)}_\sigma,\sum_\tau\phi_\ell(h^{(\ell)}_\tau), \sum_\rho\psi_\ell(h^{(\ell)}_\rho)\bigr)\in\R^H\times\R^H\times\R^H$ takes only finitely many values across all $\sigma$ in all $K\in\mathcal{F}$, because $\mathscr{C}_\ell$ and the multiplicity vectors are finite. 
Each such triple is in bijection with the triple $(c^{(\ell)}_\sigma,\{\!\{c^{(\ell)}_\tau\}\!\}_{\partial\sigma}, \{\!\{c^{(\ell)}_\rho\}\!\}_{\coface(\sigma)})$, which determines $c^{(\ell+1)}_\sigma$ by Definition~\ref{def:gswl}.  
By Definition~\ref{def:fip}, choose $\mu_\ell$ to map each observed triple to $a_{\iota_{\ell+1}(c^{(\ell+1)}_\sigma)}\in A$.

\emph{Readout.}  After $L$ rounds, $h^{(L)}_\sigma=a_{\iota_L(c^{(L)}_\sigma)}$.
Choose $\Psi$ on $\{a_1,\dots,a_{m_L}\}$ by $\Psi(a_j)=e_j\in\R^{m_L}$, the $j$-th standard basis vector, extended by finite interpolation.  Then $z(K)=\sum_\sigma\Psi(h^{(L)}_\sigma)\in\R^{m_L}$ is the histogram of round-$L$ GSWL colors on $K$, and equality of histograms across $K_1,K_2\in\mathcal{F}$ is exactly $K_1\equiv_L K_2$.
\end{proof}

\exactectfinite*

\begin{proof}
By Theorem~\ref{thm:lower}, parameters can be chosen so that $h^{(L)}_\sigma$ is injective in the observed GSWL color.  
Because $L\ge D\ge\dim(\sigma)$, Lemma~\ref{lem:coord-recovery} implies the color determines the vertex-coordinate set of $\sigma$, hence the entry time $t_\nu(\sigma)$ for every $\nu\in \mathcal{V}$.  
For each $(\nu,t)\in \mathcal{V}\times T$, the indicator $\1\{t_\nu(\sigma)\le t\}$ is a deterministic function of the hidden state, and by finite interpolation $\Psi$ can realize the vector of all $|\mathcal{V}|\cdot|T|$ such indicators jointly.  
Summing with Euler signs yields the sampled ECT.
\end{proof}

The following corollary is conditional on the recovery of entry times from the representation. 
The finite-family construction of Theorem~\ref{thm:lower} does not, by itself, establish this recovery property uniformly on an infinite class of embeddings.

\ectstability*

\begin{proof}
Since $\|\mathbf 1\{a\leq\cdot\}-\mathbf 1\{b\leq\cdot\}\|_{L^1}=|a-b|$, the triangle inequality gives
\[
  \|\operatorname{ECT}(K,x)(\nu,\cdot)
    -\widehat{\operatorname{ECT}}(\nu,\cdot)\|_{L^1}
  \leq \sum_{\sigma\in K}
    |t_\nu(\sigma)-\widehat t_\nu(\sigma)|
  \leq |K|\delta.
\]
Integrating over directions proves the bound. 
For recovered coordinates, set $\widehat t_\nu(\sigma)=\max_{v\in\sigma}\widehat x_v\cdot\nu$ and use 
\[
  |t_\nu(\sigma)-\widehat t_\nu(\sigma)|
  \leq \max_{v\in\sigma}\|x_v-\widehat x_v\|_2
  \leq\varepsilon.
\]
\end{proof}

\section{Additional experimental details}\label{app:exp_details}

All experiments were run on a MacBook Pro (November 2023) with an Apple M3 Pro chip and 18\,GB of memory. 
Synthetic and MANTRA experiments complete in under 10 minutes each; FAUST experiments with 5-fold cross-validation and augmentation take approximately 2--3 hours per model.
 
\paragraph{Architecture.}
All simplicial models use $L$ layers of boundary + coboundary message passing with hidden dimension $h_d$, LayerNorm residual connections, and ReLU activations. Global readout is mean pooling over each simplex dimension, concatenated and projected via a 2-layer MLP head. GCN and GIN models use the same depth and hidden dimension on the 1-skeleton.
From vertex coordinates alone, edge and triangle features are recoverable after one and two boundary rounds, respectively.
Dimension-wise mean and sum pooling differ by fixed per-dimension factors on fixed-template tasks, but not necessarily on MANTRA orientability.
 
\paragraph{Synthetic experiments.}
$V=40$, $E=107$, $T=68$, $h_d=32$, $L=4$, 80 epochs, Adam with $\text{lr}=10^{-3}$, cosine annealing, 75/25 train/test split.
 
\paragraph{MANTRA experiments.}
Triangulations from MANTRA~\citep{mantra2025}, \texttt{dimension=2}. For each topological type, the largest triangulation is selected as the base complex (6--12 vertices), embedded in $\R^2$ via the second and third eigenvectors of the graph Laplacian. Four deformation families, 50 samples each, same hyperparameters as synthetic.
 
\paragraph{FAUST experiments.}
100 meshes, $V=6890$, $h_d=16$, $L=3$, 200 epochs. 5-fold cross-validation with 5$\times$ augmentation (random per-axis rotation up to $0.25\mathrm{rad}$ and Gaussian noise $\sigma=0.02$). Per-sample centering and global variance normalization.

\end{document}